\documentclass[runningheads]{llncs}
\usepackage{graphicx}
\usepackage{amsmath}
\usepackage{amsfonts}
\usepackage{cite}
\usepackage{mathabx}
\usepackage{algorithm}
\usepackage{algpseudocode}
\usepackage[title]{appendix}
\usepackage{booktabs}
\usepackage{multirow}
\usepackage{xpatch} 
\usepackage{arydshln} 
\usepackage[update]{epstopdf}
\usepackage[process=all,crop=pdfcrop,cleanup={.dvi,.ps,.pdf,.log,.aux,.bbl}]{pstool}
\usepackage{relsize}
\usepackage{url}
\usepackage{xcolor}
\definecolor{webred}{rgb}{0.5,0,0}
\definecolor{webblue}{rgb}{0,0,0.8}
\usepackage[colorlinks,citecolor=webblue,linkcolor=webred,]{hyperref}

\newcommand{\R}{\mathbb{R}}
\newcommand{\Rn}{\mathbb{R}^n}
\newcommand{\coeffA}{a_1}
\newcommand{\coeffB}{a_2}
\newcommand{\coeffC}{a_3}
\newcommand{\PZ}{\mathcal{PZ}}
\newcommand{\layers}{\kappa}
\newcommand{\actFun}{\mu}
\newcommand{\fitFun}{g}
\newcommand{\precision}{\delta}
\newcommand{\neurons}{v}
\newcommand{\dist}{r}
\newcommand{\dims}{o}
\newcommand{\subsecminus}{-0.25cm}
\newcommand{\secminus}{-0.25cm}

\allowdisplaybreaks

\makeatletter
\xpatchcmd{\algorithmic}{\itemsep\z@}{\itemsep=1ex plus1pt}{}{}
\makeatother

\newcommand{\myparagraph}[1]{\vspace{-0.2cm}\paragraph{#1}\mbox{}\\ \vspace{-0.3cm}}

\begin{document}
\title{Open- and Closed-Loop Neural Network Verification using Polynomial Zonotopes}

\titlerunning{Open- and Closed-Loop Neural Network Verification using Poly. Zonotopes}

\author{Niklas Kochdumper\inst{1}, Christian Schilling\inst{2}, Matthias Althoff\inst{3}, and Stanley Bak\inst{1}}


\authorrunning{N. Kochdumper et al.}

\institute{Stony Brook University, Stony Brook, NY, USA, \\
\email{\{niklas.kochdumper,stanley.bak\}@stonybrook.edu} \and Aalborg University, Aalborg, Denmark,  \email{christianms@cs.aau.dk} \and Technichal University of Munich, Garching, Germany, \email{althoff@tum.de} \vspace{-0.5cm}}
\maketitle              
\begin{abstract}
We present a novel approach to efficiently compute tight non-convex enclosures of the image through neural networks with ReLU, sigmoid, or hyperbolic tangent activation functions. In particular, we abstract the input-output relation of each neuron by a polynomial approximation, which is evaluated in a set-based manner using polynomial zonotopes. While our approach can also can be beneficial for open-loop neural network verification, our main application is reachability analysis of neural network controlled systems, where polynomial zonotopes are able to capture the non-convexity caused by the neural network as well as the system dynamics. This results in a superior performance compared to other methods, as we demonstrate on various benchmarks.


\end{abstract}
\vspace{-0.7cm}
\section{Introduction}
\vspace{-0.1cm}

\addtocounter{footnote}{-3}

While previously artificial intelligence was mainly used for soft applications such as movie recommendations \cite{Christakou2007}, facial recognition \cite{Khan2019}, or chess computers \cite{David2016}, it is now also increasingly applied in safety-critical applications, such as autonomous driving \cite{Riedmiller2007}, human-robot collaboration \cite{Mukherjee2022}, or power system control \cite{Beaufays1994}. In contrast to soft applications, where failures usually only have minor consequences, failures in safety-critical applications in the worst case result in loss of human lives. Consequently, in order to prevent those failures, there is an urgent need for efficient methods that can verify that the neural networks used for artificial intelligence function correctly. Verification problems involving neural networks can be grouped into two main categories:
\begin{itemize}
 \setlength\itemsep{1em}
 \item \textbf{Open-loop verification:} Here the task is to check if the output of the neural network for a given input set satisfies certain properties. With this setup one can for example prove that a neural network used for image classification is robust against a certain amount of noise on the image. 
 \item \textbf{Closed-loop verification:} In this case the neural network is used as a controller for a dynamical system, e.g., to steer the system to a given goal set while avoiding unsafe regions. The safety of the controlled system can be verified using reachability analysis.  
\end{itemize} 
For both of the above verification problems, the most challenging step is to compute a tight enclosure of the image through the neural network for a given input set. Due to the high expressiveness of neural networks, their images usually have complex shapes, so that convex enclosures are often too conservative for verification. In this work, we show how to overcome this limitation with our novel approach for computing tight non-convex enclosures of images through neural networks using polynomial zonotopes. 

\vspace{\subsecminus}
\subsection{State of the Art}
\vspace{-0.1cm}

We first summarize the state of the art for open-loop neural network verification followed by reachability analysis for neural network controlled systems. Many different set representations have been proposed for computing enclosures of the image through a neural network, including intervals \cite{Wang2018}, polytopes \cite{Tran2019}, zonotopes \cite{Singh2018}, star sets \cite{Tran2019c}, and Taylor models \cite{Ivanov2021}. For neural networks with ReLU activation functions, it is possible to compute the exact image. This can be either achieved by recursively partitioning the input set into piecewise affine regions \cite{Vincent2021}, or by propagating the initial set through the network using polytopes \cite{Tran2019,Yang2021} or star sets \cite{Tran2019c}, where the set is split at all neurons that are both active or inactive. In either case the exact image is in the worst case given as a union of $2^\neurons$ convex sets, with $\neurons$ being the number of neurons in the network. To avoid this high computational complexity for exact image computation, most approaches compute a tight enclosure instead using an abstraction of the neural network. For ReLU activation functions one commonly used abstraction is the triangle relaxation \cite{Ehlers2017} (see Fig.~\ref{fig:abstractionReLU}), which can be conveniently integrated into set propagation using star sets \cite{Tran2019c}. Another possibility is to abstract the input-output relation by a zonotope (see Fig.~\ref{fig:abstractionReLU}), which is possible for ReLU, sigmoid, and hyperbolic tangent activation functions \cite{Singh2018}. One can also apply Taylor model arithmetic \cite{Makino2003} to compute the image through networks with sigmoid and hyperbolic tangent activation \cite{Ivanov2021}, which corresponds to an abstraction of the input-output relation by a Taylor series expansion. In order to better capture dependencies between different neurons, some approaches also abstract the input-output relation of multiple neurons at once \cite{Singh2019b,Mueller2021}. 

\begin{figure}[!tb] 
	\centering
	\setlength{\belowcaptionskip}{-5pt}
	\includegraphics[width = 0.99\textwidth]{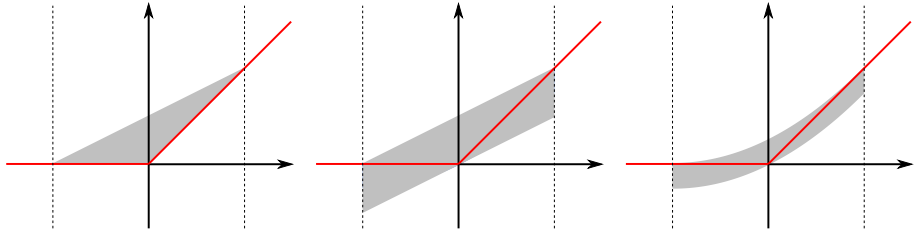}
	\vspace{-10pt}
	\caption{Triangle relaxation (left), zonotope abstraction (middle), and polynomial zonotope abstraction (right) of the ReLU activation function.}
	\label{fig:abstractionReLU}
\end{figure}

While computation of the exact image is infeasible for large networks, the enclosures obtained by abstractions are often too conservative for verification. To obtain complete verifiers, many approaches therefore use branch and bound strategies \cite{Bunel2020} that split the input set and/or single neurons until the specification can either be proven or a counterexample is found. For computational reasons branch and bound strategies are usually combined with approaches that are able to compute rough interval bounds for the neural network output very fast. Those bounds can for example be obtained using symbolic intervals \cite{Wang2018} that store linear constraints on the variables in addition to the interval bounds to preserve dependencies. The DeepPoly approach \cite{Singh2019a} uses a similar concept, but applies a back-substitution scheme to obtain tighter bounds. With the FastLin method \cite{Weng2018} linear bounds for the overall network can be computed from linear bounds for the single neurons. The CROWN approach \cite{Zhang2018b} extends this concept to linear bounds with different slopes as well as quadratic bounds. Several additional improvements for the CROWN approach have been proposed, including slope optimization using gradient descent \cite{Xu2021} and efficient ReLU splitting \cite{Wang2021}. Instead of explicitly computing the image, many approaches also aim to verify the specification directly using SMT solvers \cite{Pulina2012,Katz2017}, mixed-integer linear programming \cite{Cheng2017,Tjeng2019}, semidefinite programming \cite{Raghunathan2018}, and convex optimization \cite{Khedr2021}. 

For reachability analysis of neural network controlled systems one has to compute the set of control inputs in each control cycle, which is the image of the current reachable set through the neural network controller. Early approaches compute the image for ReLU networks exactly using polytopes \cite{Xiang2018b} or star sets \cite{Tran2019b}. Since in this case the number of coexisting sets grows rapidly over time, these approaches have to unite sets using convex hulls \cite{Xiang2018b} or interval enclosures \cite{Tran2019b}, which often results in large over-approximations. If template polyhedra are used as a set representation, reachability analysis for neural network controlled systems with discrete-time plants reduces to the task of computing the maximum output along the template directions \cite{Dutta2018b}, which can be done efficiently. Neural network controllers with sigmoid and hyperbolic tangent activation functions can be converted to an equivalent hybrid automaton \cite{Ivanov2019}, which can be combined with the dynamics of the plant using the automaton product. However, since each neuron is represented by an additional state, the resulting hybrid automaton is very high-dimensional, which makes reachability analysis challenging. Some approaches approximate the overall network with a polynomial function \cite{Dutta2019b,Huang2019} using polynomial regression based on samples \cite{Dutta2019b} and Bernstein polynomials \cite{Huang2019}. Yet another class of methods \cite{Claviere2021,Ivanov2021,Schilling2022,Tran2020b} employs abstractions of the input-output relation for the neurons to compute the set of control inputs using intervals \cite{Claviere2021}, star sets \cite{Tran2020b}, Taylor models\cite{Ivanov2021}, and a combination of zonotopes and Taylor models\cite{Schilling2022}. Common tools for reachability analysis of neural network controlled systems are JuliaReach \cite{Bogomolov2019a}, NNV \cite{Tran2020b}, POLAR \cite{Huang2022}, ReachNN* \cite{Fan2020}, RINO \cite{Goubault2022}, Sherlock \cite{Dutta2019}, Verisig \cite{Ivanov2019}, and Verisig 2.0 \cite{Ivanov2021}, where JuliaReach uses zonotopes for neural network abstraction \cite{Schilling2022}, NVV supports multiple set representations, ReachNN* applies the Bernstein polynomial method \cite{Huang2019}, POLAR approximates single neurons by Bernstein polynomials \cite{Huang2022}, RINO computes interval inner- and outer-approximations \cite{Goubault2022}, Sherlock uses the polynomial regression approach \cite{Dutta2019b}, Verisig performs the conversion to a hybrid automaton \cite{Ivanov2019}, and Verisig 2.0 uses the Taylor model based neural network abstraction method \cite{Ivanov2021}.

\vspace{\subsecminus}
\subsection{Overview}
\vspace{-3pt}

In this work, we present a novel approach for computing tight non-convex enclosures of images through neural networks with ReLU, sigmoid, or hyperbolic tangent activation functions. The high-level idea is to approximate the input-output relation of each neuron by a polynomial function, which results in the abstraction visualized in Fig.~\ref{fig:abstractionReLU}. Since polynomial zonotopes are closed under polynomial maps, the image through this function can be computed exactly, yielding a tight enclosure of the image through the overall neural network. The remainder of this paper is structured as follows: After introducing some preliminaries in Sec.~\ref{sec:preliminaries}, we present our approach for computing tight enclosures of images through neural networks in Sec.~\ref{sec:image}. Next, we show how to utilize this result for reachability analysis of neural network controlled systems in Sec.~\ref{sec:reach}. Afterwards, in Sec.~\ref{sec:operations}, we introduce some special operations on polynomial zonotopes that we require for image and reachable set computation, before we finally demonstrate the performance of our approach on numerical examples in Sec.~\ref{sec:numEx}.

\vspace{-2pt}
\vspace{\subsecminus}
\subsection{Notation}
\vspace{-3pt}

Sets are denoted by calligraphic letters, matrices by uppercase letters, and vectors by lowercase letters. Given a vector $b \in \mathbb{R}^n$, $b_{(i)}$ refers to the $i$-th entry. Given a matrix $A \in \mathbb{R}^{\dims \times n}$, $A_{(i,\cdot)}$ represents the $i$-th matrix row, $A_{(\cdot,j)}$ the $j$-th column, and $A_{(i,j)}$ the $j$-th entry of matrix row $i$. The concatenation of two matrices $C$ and $D$ is denoted by $[C~D]$, and $I_n \in \R^{n \times n}$ is the identity matrix. The symbols $\mathbf{0}$ and $\mathbf{1}$ represent matrices of zeros and ones of proper dimension, the empty matrix is denoted by $[~]$, and $\mathrm{diag}(a)$ returns a diagonal matrix with $a \in \R^n$ on the diagonal. Given a function $f(x)$ defined as $f:~\R \to \R$, $f'(x)$ and $f''(x)$ denote the first and second derivative with respect to $x$. The left multiplication of a matrix $A \in \mathbb{R}^{\dims \times n}$ with a set $\mathcal{S} \subset \mathbb{R}^n$ is defined as $A \, \mathcal{S} := \{ A\, s ~ | ~ s \in \mathcal{S} \}$, the Minkowski addition of two sets $\mathcal{S}_1 \subset \mathbb{R}^n$ and $\mathcal{S}_2 \subset \mathbb{R}^n$ is defined as $\mathcal{S}_1 \oplus \mathcal{S}_2 := \{ s_1 + s_2 ~|~ s_1 \in \mathcal{S}_1, s_2 \in \mathcal{S}_2 \}$, and the Cartesian product of two sets $\mathcal{S}_1 \subset \mathbb{R}^n$ and $\mathcal{S}_2 \subset \mathbb{R}^\dims$ is defined as $\mathcal{S}_1 \times \mathcal{S}_2 := \big\{ [s_1^T~s_2^T]^T ~|~ s_1 \in \mathcal{S}_1, s_2 \in \mathcal{S}_2 \big\}$. We further introduce an $n$-dimensional interval as $\mathcal{I} := [l,u],~ \forall i ~ l_{(i)} \leq u_{(i)},~ l,u \in \mathbb{R}^n$.


\vspace{\secminus}
\section{Preliminaries} \label{sec:preliminaries}
\vspace{-3pt}

Let us first introduce some preliminaries required throughout the paper. While the concepts presented in this work can equally be applied to more complex network architectures, we focus on feed-forward neural networks for simplicity:

\begin{definition}
(Feed-forward neural network) A feed-forward neural network with $\layers$ hidden layers consists of weight matrices $W_i \in \R^{\neurons_i \times \neurons_{i-1}}$ and bias vectors $b_i \in \R^{\neurons_i}$ with $i \in \{1,\dots,\layers+1\}$ and $\neurons_i$ denoting the number of neurons in layer $i$. The output $y \in \R^{\neurons_{\layers+1}}$ of the neural network for the input $x \in \R^{\neurons_0}$ is
\begin{equation*}
	y := y_{\layers+1} ~~ \text{with} ~~ y_0 = x,~~ y_{i(j)} = \actFun \bigg (\sum_{k=1}^{\neurons_{i-1}} W_{i(j,k)} \, y_{i-1(k)} + b_{i(j)} \bigg), ~ i = 1,\dots,\layers+1,
\end{equation*}
where $\actFun:~\R \to \R$ is the activation function.
\label{def:neuralNetwork}
\end{definition}
In this paper we consider ReLU activations $\actFun(x) = \max(0,x)$, sigmoid activations $\actFun(x) = \sigma(x) = 1/(1+e^{-x})$, and hyperbolic tangent activations $\actFun(x) = \tanh(x) = (e^x - e^{-x})/(e^x + e^{-x})$. Moreover, neural networks often do not apply activation functions on the output neurons, which corresponds to using the identity map $\actFun(x) = x$ for the last layer. The image $\mathcal{Y}$ through a neural network is defined as the set of outputs for a given set of inputs $\mathcal{X}_0$, which is according to Def.~\ref{def:neuralNetwork} given as 
\begin{equation*}
	\mathcal{Y} = \bigg\{ y_{\layers+1} ~\bigg|~ y_0 \in \mathcal{X}_0,~\forall i \in \{1,\dots,\layers+1\}:~ y_{i(j)} = \actFun \bigg (\sum_{k=1}^{\neurons_{i-1}} W_{i(j,k)} \, y_{i-1(k)} + b_{i(j)} \bigg)\bigg\}.
\end{equation*}
We present a novel approach for tightly enclosing the image through a neural network by a polynomial zonotope \cite{Althoff2013a}, where we use the sparse representation of polynomial zonotopes \cite{Kochdumper2019}\footnote{In contrast to \cite[Def.~1]{Kochdumper2019}, we explicitly do not integrate the constant offset $c$ in $G$. Moreover, we omit the identifier vector used in \cite{Kochdumper2019} for simplicity}:

\begin{figure}[!tb] \label{fig:examplePolyZono}
	\centering
	\includegraphics[width = \textwidth]{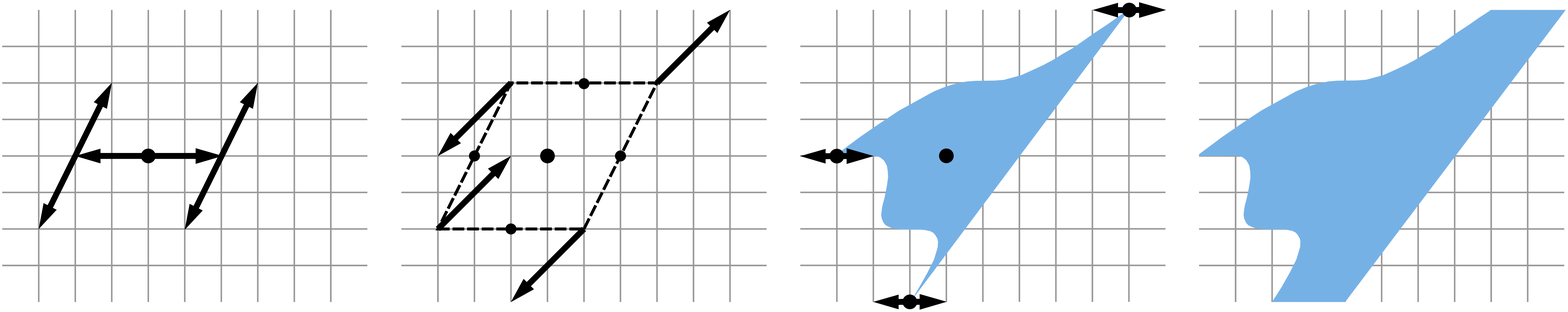}
	\vspace{-20pt}
	\caption{Step-by-step construction of the polynomial zonotope from Example~\ref{ex:PolyZonotope}.}
\end{figure}

\begin{definition}
  (Polynomial zonotope) Given a constant offset $c \in \mathbb{R}^n$, a generator matrix of dependent generators $G \in \mathbb{R}^{n \times h}$, a generator matrix of independent generators $G_I \in \mathbb{R}^{n \times q}$, and an exponent matrix $E \in \mathbb{N}_{0}^{p \times h}$, a polynomial zonotope $\mathcal{PZ} \subset \Rn$ is defined as  
  \begin{equation*}
    \mathcal{PZ} := \bigg\{ c + \sum _{i=1}^h \bigg( \prod _{k=1}^p \alpha _k ^{E_{(k,i)}} \bigg) G_{(\cdot,i)} + \sum _{j=1}^{q} \beta _j G_{I(\cdot,j)} ~ \bigg| ~ \alpha_k, \beta_j \in [-1,1] \bigg\}.
  \end{equation*}
  The scalars $\alpha_k$ are called \textit{ dependent factors} since a change in their value affects multiplication with multiple generators. Analogously, the scalars $\beta_j$ are called \textit{independent factors} because they only affect the multiplication with one generator.
For a concise notation we use the shorthand $\mathcal{PZ} = \langle c,G, G_I, E \rangle_{PZ}$.
  \label{def:polyZono}
\end{definition}  
Let us demonstrate polynomial zonotopes by an example:

\begin{example} 
	The polynomial zonotope
	\begin{equation*}
		\mathcal{PZ} = \left\langle \begin{bmatrix} 4 \\ 4 \end{bmatrix}, \begin{bmatrix} 2 & ~1~ & 2 \\ 0 & ~2~ & 2 \end{bmatrix}, \begin{bmatrix} 1 \\ 0 \end{bmatrix}, \begin{bmatrix} 1 & ~0~ & 3 \\ 0 & 1 & 1 \end{bmatrix} \right\rangle_{PZ}
	\end{equation*}
	defines the set
	\begin{equation*}
  		\mathcal{PZ} = \bigg\{ \begin{bmatrix} 4 \\ 4 \end{bmatrix} + \begin{bmatrix} 2 \\ 0 	\end{bmatrix} \alpha_1 + \begin{bmatrix} 1 \\ 2 \end{bmatrix} \alpha_2 + \begin{bmatrix} 2 \\ 2 \end{bmatrix} \alpha_1^3 \alpha_2 + \begin{bmatrix} 1 \\ 0 \end{bmatrix} \beta_1 ~ \bigg| ~ \alpha_1, \alpha_2, \beta_1 \in [-1,1] \bigg\}.
	\end{equation*}
The construction of this polynomial zonotope is visualized in Fig.~\ref{fig:examplePolyZono}. 
	\label{ex:PolyZonotope}
\end{example}


\section{Image Enclosure} \label{sec:image}

We now present our novel approach for computing tight non-convex enclosures of images through neural networks. The general concept is to approximate the input-output relation of each neuron by a polynomial function, the image through which can be computed exactly since polynomial zonotopes are closed under polynomial maps. For simplicity, we focus on quadratic approximations here, but the extension to polynomials of higher degree is straightforward. 

The overall procedure for computing the image is summarized in Alg.~\ref{alg:image}, where the computation proceeds layer by layer. For each neuron in the current layer $i$ we first calculate the corresponding input set in Line~\ref{line:affineMap}. Next, in Line~\ref{line:bounds2}, we compute a lower and an upper bound for the input to the neuron. Using these bounds we then calculate a quadratic approximation for the neuron's input-output relation in Line~\ref{line:approx}. This approximation is evaluated in a set-based manner in Line~\ref{line:quadMap}. The resulting polynomial zonotope $\langle c_q,G_q,G_{I,q},E_q\rangle_{PZ}$ forms the $j$-th dimension of the set $\PZ$ representing the output of the whole layer (see Line~\ref{line:stack} and Line \ref{line:minkSum}). To obtain a formally correct enclosure, we have to account for the error made by the approximation. We therefore compute the difference between the activation function and the quadratic approximation in Line~\ref{line:diff} and add the result to the output set in Line \ref{line:minkSum}. By repeating this procedure for all layers, we finally obtain a tight enclosure of the image through the neural network. A demonstrating example for Alg.~\ref{alg:image} is shown in Fig.~\ref{fig:neuralNetwork}.

For ReLU activations the quadratic approximation only needs to be calculated if $l < 0 \wedge u > 0$ since we can use the exact input-output relations $\fitFun(x) = x$ and $\fitFun(x) = 0$ if $l \geq 0$ or $u \leq 0$ holds. Due to the evaluation of the quadratic map defined by $\fitFun(x)$, the representation size of the polynomial zonotope $\PZ$ increases in each layer. For deep neural networks it is therefore advisable to repeatedly reduce the representation size after each layer using order reduction \cite[Prop.~16]{Kochdumper2019}. Moreover, one can also apply the $\texttt{compact}$ operation described in \cite[Prop.~2]{Kochdumper2019} after each layer to remove potential redundancies from $\PZ$. Next, we explain the approximation of the input-output relation as well as the computation of the approximation error in detail. 

\begin{algorithm}[!tb]
	\caption{Enclosure of the image through a neural network} \label{alg:image}
	{\raggedright \textbf{Require:} Neural network with weight matrices $W_i$ and bias vectors $b_i$, initial set $\mathcal{X}_0$.
	
	\textbf{Ensure:} Tight enclosure $\mathcal{PZ} \supseteq \mathcal{Y}$ of the image $\mathcal{Y}$.
	
	}
	\begin{algorithmic}[1]
		\State $\PZ \gets \mathcal{X}_0$	
		\For{$i \gets 1$ to $\layers +1$} \hfill (loop over all layers)
			\State $c \gets \mathbf{0},~G \gets \mathbf{0}, ~G_I \gets \mathbf{0},~\underline{d} \gets \mathbf{0},~\overline{d} \gets \mathbf{0}$ 
			\For{$j \gets 1$ to $\neurons_i$} \hfill (loop over all neurons in the layer)
				\State $\PZ_j \gets W_{i(j,\cdot)} \PZ \oplus b_{i(j)}$ \label{line:affineMap} \hfill (map with weight matrix and bias using \eqref{eq:affineMap})
				\State $l,u \gets$ lower and upper bound for $\mathcal{PZ}_j$ according to Prop.~\ref{prop:intervalEnclosure} \label{line:bounds2}
				\State $\fitFun(x) = \coeffA \,x^2 + \coeffB \,x + \coeffC \gets $ quad. approx. on $[l,u]$ according to Sec.~\ref{sec:actFunApprox} \label{line:approx}
				\State $\langle c_q,G_q,G_{I,q},E_q\rangle_{PZ} \gets$ image of $\mathcal{PZ}_j$ through $\fitFun(x)$ according to Prop.~\ref{prop:quadMap} \label{line:quadMap}
				\State $c_{(j)} \gets c_q, ~ G_{(j,\cdot)} \gets G_q, ~ G_{I(j,\cdot)} \gets G_{I,q}, ~ E \gets E_q$ \label{line:stack} \hfill (add to output set)
				\State $\underline{d}_{(j)},\overline{d}_{(j)} \gets$ difference between $\fitFun(x)$ and activation function acc. to Sec.~\ref{sec:bounds} \label{line:diff}
			\EndFor
			\State $\PZ \gets \langle c,G,G_I,E \rangle_{PZ} \oplus [\underline{d},\overline{d}]$ \label{line:minkSum} \hfill (add approximation error using \eqref{eq:minkSumInt})
		\EndFor
	\end{algorithmic}
\end{algorithm}

\vspace{\subsecminus}
\subsection{Activation Function Approximation}
\label{sec:actFunApprox}

The centerpiece of our algorithm for computing the image of a neural network is the approximation of the input-output relation defined by the activation function $\actFun(x)$ with a quadratic expression $\fitFun(x)= \coeffA \,x^2 + \coeffB \,x + \coeffC$ (see Line~\ref{line:approx} of Alg.~\ref{alg:image}). In this section we present multiple possibilities to obtain good approximations. 

\myparagraph{Polynomial Regression}

\noindent For polynomial regression we uniformly select $N$ samples $x_i$ from the interval $[l,u]$ and then determine the polynomial coefficients $\coeffA,\coeffB,\coeffC$ by minimizing the average squared distance between the activation function and the quadratic approximation:
\begin{equation}
	\min_{\coeffA,\coeffB,\coeffC} \frac{1}{N} \sum_{i=1}^N \big(\actFun(x_i) - \coeffA \,x_i^2 - \coeffB \,x_i - \coeffC\big)^2.
	\label{eq:polyReg}
\end{equation}
It is well known that the optimal solution to \eqref{eq:polyReg} is  
\begin{equation*}
	\begin{bmatrix} \coeffA \\ \coeffB \\ \coeffC \end{bmatrix} = A^\dagger b ~~ \text{with} ~~ A = \begin{bmatrix} x_1^2 & x_1 & 1 \\ \vdots & \vdots & \vdots \\ x_N^2 & x_N & 1 \end{bmatrix}, ~~ b = \begin{bmatrix} \actFun(x_1) \\ \vdots \\ \actFun(x_N) \end{bmatrix},
\end{equation*}
where $A^\dagger = (A^T A)^{-1} A^T$ is the Moore-Penrose inverse of matrix A. For the numerical experiments in this paper we use $N=10$ samples.

\myparagraph{Closed-Form Expression}

\noindent For ReLU activations a closed-form expression for a quadratic approximation can be obtained by enforcing the conditions $\fitFun(l) = 0$, $\fitFun'(l) = 0$, and $\fitFun(u) = u$. The solution to the corresponding equation system $\coeffA \,l^2 + \coeffB \,l + \coeffC = 0$, $2 \coeffA l + \coeffB = 0$, $\coeffA \,u^2 + \coeffB \,u + \coeffC = u$ is
\begin{equation*}
	\coeffA = \frac{u}{(u-l)^2}, ~~ \coeffB = \frac{-2l u}{(u-l)^2}, ~~ \coeffC = \frac{u^2(2l-u)}{(u-l)^2} + u,
\end{equation*}
which results in the enclosure visualized in Fig.~\ref{fig:abstractionReLU}. This closed-form expression is very precise if the interval $[l,u]$ is close to being symmetric with respect to the origin ($|l| \approx |u|$), but becomes less accurate if one bound is significantly larger than the other ($|u| \gg |l|$ or $|l| \gg |u|$).

\myparagraph{Taylor Series Expansion}

\noindent For sigmoid and hyperbolic tangent activation functions a quadratic fit can be obtained using a second-order Taylor series expansion of the activation function $\actFun(x)$ at the expansion point $x^* = 0.5(l+u)$: 
\begin{equation*}
\begin{split}
	& \actFun(x) \approx \actFun(x^*) + \actFun'(x^*)(x - x^*) + 0.5 \, \actFun''(x^*)(x-x^*)^2 = \\
	& ~\underbrace{0.5 \, \actFun''(x^*)}_{\coeffA} x^2 + \big( \underbrace{\actFun'(x^*) - \actFun''(x^*)\, x^*}_{\coeffB}\big) x + \underbrace{\actFun(x^*) - \actFun'(x^*)x^* + 0.5 \, \actFun''(x^*) \, {x^*}^2}_{ \coeffC}, 
\end{split}
\end{equation*}
where the derivatives for sigmoid activations are $\actFun'(x) = \sigma(x)(1-\sigma(x))$ and $\actFun''(x) = \sigma(x)(1-\sigma(x))(1-2\sigma(x))$, and the derivatives for hyperbolic tangent activations are $\actFun'(x) = 1-\tanh(x)^2$ and $\actFun''(x) = -2 \tanh(x)(1-\tanh(x)^2)$. The Taylor series expansion method is identical to the concept used in \cite{Ivanov2021}.

\myparagraph{Linear Approximation}

\noindent Since a linear function represents a special case of a quadratic function, Alg.~\ref{alg:image} can also be used in combination with linear approximations. Such approximations are provided by the zonotope abstraction in \cite{Singh2018}. Since closed-form expressions for the bounds $\underline{d}$ and $\overline{d}$ of the approximation error are already specified in \cite{Singh2018}, we can omit the error bound computation described in Sec.~\ref{sec:bounds} in this case. For ReLU activations we obtain according to \cite[Theorem~3.1]{Singh2018}
\begin{equation*}
	\coeffA = 0, ~~ \coeffB = \frac{u}{u-l}, ~~ \coeffC = \frac{-u \, l}{2(u-l)}, ~~ \underline{d} = \frac{-u \, l}{2(u-l)}, ~~ \overline{d} = \frac{u \, l}{2(u-l)},
\end{equation*}
which results in the zonotope enclosure visualized in Fig.~\ref{fig:abstractionReLU}. For sigmoid and hyperbolic tangent activations we obtain according to \cite[Theorem~3.2]{Singh2018}
\begin{equation*}
\begin{split}
	& \coeffA = 0,~~\coeffB = \min(\actFun'(l),\actFun'(u)),~~\coeffC = 0.5 (\actFun(u) + \actFun(l) - \coeffB(u+l)), \\
	& \underline{d} = 0.5 (\actFun(u) - \actFun(l) - \coeffB(u-l)),~~ \overline{d} = -0.5( \actFun(u) - \actFun(l) - \coeffB(u-l)),
\end{split}
\end{equation*}
where the derivatives of the sigmoid function and the hyperbolic tangent are specified in the paragraph above.

\begin{figure}[!tb] 
	\centering
	\includegraphics[width = 0.9\textwidth]{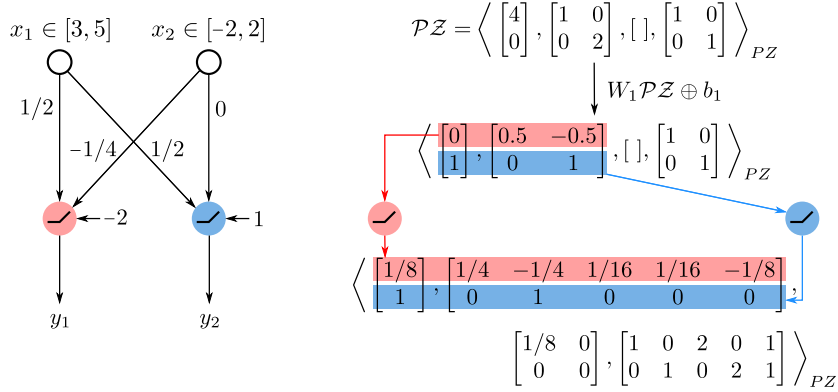}
	\caption{Exemplary neural network with ReLU activations (left) and the corresponding image enclosure computed with polynomial zonotopes (right), where we use the approximation $\fitFun(x) = 0.25\,x^2 + 0.5\,x + 0.25$ for the red neuron and the approximation $\fitFun(x) = x$ for the blue neuron.}
	\label{fig:neuralNetwork}
\end{figure}

We observed from experiments that for ReLU activations the closed-form expression usually results in a tighter enclosure of the image than polynomial regression. For sigmoid and hyperbolic tangent activations, on the other hand, polynomial regression usually performs better than the Taylor series expansion. It is also possible to combine multiple of the methods described above by executing them in parallel and selecting the one that results in the smallest approximation error $[\underline{d},\overline{d}]$. 
Since the linear approximation does not increase the number of generators, it represents an alternative to order reduction when dealing with deep neural networks. Here, the development of a method to decide automatically for which layers to use a linear and for which a quadratic approximation is a promising direction for future research.

\vspace{-2pt}
\vspace{\subsecminus}
\subsection{Bounding the Approximation Error}
\label{sec:bounds}
\vspace{-3pt}

To obtain a sound enclosure we need to compute the difference between the activation function $\actFun(x)$ and the quadratic approximation $\fitFun(x) = \coeffA \,x^2 + \coeffB \,x + \coeffC$ from Sec.~\ref{sec:actFunApprox} on the interval $[l,u]$. In particular, this corresponds to determining
\begin{equation*}
	\underline{d} = \min_{x \in [l,u]} \underbrace{\actFun(x) - \coeffA \,x^2 - \coeffB \,x - \coeffC}_{d(x)} ~~\text{and}~~ \overline{d} = \max_{x \in [l,u]} \underbrace{\actFun(x) - \coeffA \,x^2 - \coeffB \,x - \coeffC}_{d(x)}.
\end{equation*}
Depending on the type of activation function, we use different methods for this.  

\myparagraph{Rectified Linear Unit (ReLU)}

\noindent For ReLU activation functions we split the interval $[l,u]$ into the two intervals $[l,0]$ and $[0,u]$ on which the activation function is constant and linear, respectively. On the interval $[l,0]$ we have $d(x) = - \coeffA \,x^2 - \coeffB \,x - \coeffC$, and on the interval $[0,u]$ we have $d(x) = - \coeffA \,x^2 + (1-\coeffB) \,x - \coeffC$. In both cases $d(x)$ is a quadratic function whose maximum and minimum values are either located on the interval boundary or at the point $x^*$ where the derivative of $d(x)$ is equal to zero. The lower bound on $[l,0]$ is therefore given as $\underline{d} = \min(d(l),d(x^*),d(0))$ if $x^* \in [l,0]$ and $\underline{d} = \min(d(l),d(0))$ if $x^* \not\in [l,0]$, where $x^* = -0.5\, \coeffB/\coeffA$. The upper bound as well as the bounds for $[0,u]$ are computed in a similar manner. Finally, the overall bounds are obtained by taking the minimum and maximum of the bounds for the intervals $[l,0]$ and $[0,u]$. 

\myparagraph{Sigmoid and Hyperbolic Tangent}

\noindent Here our high-level idea is to sample the function $d(x)$ at points $x_i$ with distance $\Delta x$ distributed uniformly over the interval $[l,u]$. From rough bounds for the derivative $d'(x)$ we can then deduce how much the function value between two sample points changes at most, which yields tight bounds $\overline{d}_b \geq \overline{d}$ and $\underline{d}_b \leq \underline{d}$. In particular, we want to choose the sampling rate $\Delta x$ such that the bounds $\overline{d}_b,\underline{d}_b$ comply to a user-defined precision $\precision > 0$:
\begin{equation}
	\overline{d} + \precision \geq \overline{d}_b \geq \overline{d} ~~ \text{and} ~~ \underline{d} -\precision \leq \underline{d}_b \leq \underline{d}.
	\label{eq:condSampling}
\end{equation}
We observe that for both, sigmoid and hyperbolic tangent, the derivative is globally bounded by $\actFun'(x) \in [0,\overline{\actFun}]$, where $\overline{\actFun} = 0.25$ for the sigmoid and $\overline{\actFun} = 1$ for the hyperbolic tangent. In addition, it holds that the derivative of the quadratic approximation $\fitFun(x) = \coeffA \,x^2 + \coeffB \,x + \coeffC$ is bounded by $\fitFun'(x) \in [\underline{\fitFun},\overline{\fitFun}]$ on the interval $[l,u]$, where $\underline{\fitFun} = \min(2\coeffA l + \coeffB,2\coeffA u + \coeffB)$ and $\overline{\fitFun} = \max(2\coeffA l + \coeffB,2\coeffA u + \coeffB)$. As a consequence, the derivative of the difference $d(x) = \actFun(x) - \fitFun(x)$ is bounded by $d'(x) \in [-\overline{\fitFun},\overline{\actFun} - \underline{\fitFun}]$. The value of $d(x)$ can therefore at most change by $\pm \Delta d$ between two samples $x_i$ and $x_{i+1}$, where $\Delta d = \Delta x \max(|-\overline{\fitFun}|,|\overline{\actFun} -\overline{\fitFun}|)$. To satisfy \eqref{eq:condSampling} we require $\Delta d \leq \precision$, so that we have to choose the sampling rate as $\Delta x \leq \precision / \max(|-\overline{\fitFun}|,|\overline{\actFun} -\overline{\fitFun}|)$. Finally, the bounds are computed by taking the maximum and minimum of all samples: $\overline{d}_b = \max_{i} d(x_i) + \precision$ and $\underline{d}_b = \min_{i} d(x_i) - \precision$. For our experiments we use a precision of $\precision = 0.001$.


\vspace{\secminus}
\section{Neural Network Controlled Systems} \label{sec:reach}

Reachable sets for neural network controlled systems can be computed efficiently by combining our novel image enclosure approach for neural networks with a reachability algorithm for nonlinear systems. We consider general nonlinear systems
\begin{equation}
	\dot x (t) = f\big(x(t),u_c(x(t),t),w(t)\big),
	\label{eq:nonlinSys}
\end{equation}
where $x \in \Rn$ is the system state, $u_c:~\Rn \times \R \to \R^m$ is a control law, $w(t) \in \mathcal{W} \subset \R^\dist$ is a vector of uncertain disturbances, and $f:~\mathbb{R}^n \times \mathbb{R}^m \times \R^\dist \to \mathbb{R}^n$ is a Lipschitz continuous function. For neural network controlled systems the control law $u_c(x(t),t)$ is given by a neural network. Since neural network controllers are usually realized as digital controllers, we consider the sampled-data case where the control input is only updated at discrete times $t_0,t_0+\Delta t,t_0 + 2\Delta t, \dots, t_F$ and kept constant in between. Here, $t_0$ is the initial time, $t_F$ is the final time, and $\Delta t$ is the sampling rate. Without loss of generality, we assume from now on that $t_0 = 0$ and $t_F$ is a multiple of $\Delta t$. The reachable set is defined as follows:
\begin{definition}
	(Reachable set) Let $\xi(t,x_0,u_c(\cdot),w(\cdot))$ denote the solution to \eqref{eq:nonlinSys} for initial state $x_0 = x(0)$, control law $u_c(\cdot)$, and the disturbance trajectory $w(\cdot)$. The reachable set for an initial set $\mathcal{X}_0 \subset \mathbb{R}^n$ and a disturbance set $\mathcal{W} \subset \mathbb{R}^\dist$ is 
	\begin{equation*}
		\mathcal{R}(t) := \big\{ \xi(t,x_0,u_c(\cdot),w(\cdot)) ~\big |~ x_0 \in \mathcal{X}_0, \forall t^* \in [0,t]:~ w(t^*) \in \mathcal{W} \big \}.
	\end{equation*} 
\end{definition}
Since the exact reachable set cannot be computed for general nonlinear systems, we compute a tight enclosure instead. We exploit that the control input is piecewise constant, so that the reachable set for each control cycle can be computed using the extended system
\begin{equation}
	\begin{bmatrix} \dot x(t) \\ \dot u(t) \end{bmatrix} = \begin{bmatrix} f(x(t),u(t),w(t)) \\ \mathbf{0} \end{bmatrix}
	\label{eq:extSys}
\end{equation}
together with the initial set $\mathcal{X}_0 \times \mathcal{Y}$, where $\mathcal{Y}$ is the image of $\mathcal{X}_0$ through the neural network controller. The overall algorithm is specified in Alg.~\ref{alg:reach}. Its high-level concept is to loop over all control cycles, where in each cycle we first compute the image of the current reachable set through the neural network controller in Line~\ref{line:image}. Next, the image is combined with the reachable set using the Cartesian product in Line~\ref{line:cartProd}. This yields the initial set for the extended system in \eqref{eq:extSys}, for which we compute the reachable set $\widehat{\mathcal{R}}(t_{i+1})$ at time $t_{i+1}$ as well as the reachable set $\widehat{\mathcal{R}}(\tau_i)$ for the time interval $\tau_i$ in Line~\ref{line:reach}. While it is possible to use arbitrary reachability algorithms for nonlinear systems,  we apply the conservative polynomialization algorithm \cite{Althoff2013a} since it performs especially well in combination with polynomial zonotopes. Finally, in Line~\ref{line:project}, we project the reachable set back to the original system dimensions. 

\begin{algorithm}[!tb]
	\caption{Reachable set for a neural network controlled system} \label{alg:reach}
	{\raggedright \textbf{Require:} Nonlinear system $\dot x(t) = f(x(t),u_c(x(t),t),w(t))$, neural network controller $u_c(x(t),t)$, initial set $\mathcal{X}_0$, disturbance set $\mathcal{W}$, final time $t_F$, sampling rate $\Delta t$.
	
	\textbf{Ensure:} Tight enclosure $\mathcal{R} \supseteq \mathcal{R}([0,t_F])$ of the reachable set $\mathcal{R}([0,t_F])$.
	
	}
	\begin{algorithmic}[1]
		\State $t_0 \gets 0$, $\mathcal{R}(t_0) \gets \mathcal{X}_0$	
		\For{$i \gets 0$ to $t_F/\Delta t - 1$} \hfill (loop over all control cycles)
			\State $\mathcal{Y} \gets$ image of $\mathcal{R}(t_i)$ through the neural network controller using Alg.~\ref{alg:image} \label{line:image}
			\State $\widehat{\mathcal{R}}(t_i) \gets \mathcal{R}(t_i) \times \mathcal{Y}$ \label{line:cartProd} \hfill (combine reachable set and input set using \eqref{eq:cartProd}) 
			\State $t_{i+1} \gets t_i + \Delta t$, $\tau_{i} \gets [t_{i},t_{i+1}]$ \hfill (update time)
			\State $\widehat{\mathcal{R}}(t_{i+1}),\widehat{\mathcal{R}}(\tau_i) \gets$ reachable set for extended system in \eqref{eq:extSys} starting from $\widehat{\mathcal{R}}(t_i)$ \label{line:reach}
			\State $\mathcal{R}(t_{i+1}) \gets [I_n~\mathbf{0}]\, \widehat{\mathcal{R}}(t_{i+1})$, $\mathcal{R}(\tau_{i}) \gets [I_n~\mathbf{0}]\, \widehat{\mathcal{R}}(\tau_i)$ \label{line:project} \hfill (projection using \eqref{eq:affineMap}) 
		\EndFor
		\State $\mathcal{R} \gets \bigcup_{i=0}^{t_F/\Delta t-1} \mathcal{R}(\tau_i)$ \hfill (reachable set for the whole time horizon)
	\end{algorithmic}
\end{algorithm}


\vspace{\secminus}
\section{Operations on Polynomial Zonotopes} \label{sec:operations}
\vspace{-0.1cm}

Alg.~\ref{alg:image} and Alg.~\ref{alg:reach} both require some special operations on polynomial zonotopes, the implementation of which we present now. Given a polynomial zonotope $\mathcal{PZ} = \langle c,G,G_I,E \rangle_{PZ} \subset \Rn$, a matrix $A \in \R^{o \times n}$, a vector $b \in \R^{o}$, and an interval $\mathcal{I} = [l,u] \subset \R^n$, the affine map and the Minkowski sum with an interval are given as
\begin{align}
	& A \, \mathcal{PZ} \oplus b = \langle Ac + b, AG, AG_I, E\rangle_{PZ} \label{eq:affineMap} \\
	& \mathcal{PZ} \oplus \mathcal{I} = \langle c + 0.5(u+l), G, [G_I~0.5\,\text{diag}(u-l)],E\rangle_{PZ}, \label{eq:minkSumInt}
\end{align}
which follows directly from \cite[Prop.~8]{Kochdumper2019}, \cite[Prop.~9]{Kochdumper2019}, and \cite[Prop. 2.1]{Althoff2010a}. For the Cartesian product used in Line~\ref{line:cartProd} of Alg.~\ref{alg:reach} we can exploit the special structure of the sets to calculate the Cartesian product of two polynomial zonotopes  $\mathcal{PZ}_1 = \langle c_1,G_1, G_{I,1},E_1 \rangle_{PZ} \subset \Rn$ and $\mathcal{PZ}_2 = \langle c_2,[G_2~\widehat{G}_2], [G_{I,2}~\widehat{G}_{I,2}],[E_1 ~E_2] \rangle_{PZ} \subset \R^{\dims}$ as
	\begin{equation}
		\PZ_1 \times \PZ_2 = \bigg \langle \begin{bmatrix} c_1 \\ c_2 \end{bmatrix},\begin{bmatrix} G_1 & \mathbf{0} \\ G_2 & \widehat{G}_2 \end{bmatrix},\begin{bmatrix} G_{I,1} & \mathbf{0} \\ G_{I,2} & \widehat{G}_{I,2} \end{bmatrix},[E_1 ~ E_2] \bigg \rangle_{PZ}. \label{eq:cartProd}
	\end{equation}
	In contrast to \cite[Prop.~11]{Kochdumper2019}, this implementation of the Cartesian product explicitly preserves dependencies between the two sets, which is possible since both polynomial zonotopes have identical dependent factors. Computing the exact bounds of a polynomial zonotope in Line~\ref{line:bounds2} of Alg.~\ref{alg:image} would be computationally infeasible, especially since this has to be done for each neuron in the network. We therefore compute a tight enclosure of the bounds instead, which can be done very efficiently: 

\begin{proposition} \label{prop:intervalEnclosure}
	(Interval enclosure) Given a polynomial zonotope $\PZ = \langle c,G,\linebreak[3] G_I,E \rangle_{PZ} \subset \Rn$, an enclosing interval can be computed as
	\begin{equation*}	
		\mathcal{I} = [c + g_1 - g_2 - g_3 - g_4, c + g_1 + g_2 + g_3 + g_4] \supseteq \PZ
	\end{equation*}
	with
	\begin{equation*}
	\begin{split}
		& g_1 = 0.5 \sum_{i \in \mathcal{H}} G_{(\cdot,i)}, ~g_2 = 0.5 \sum_{i \in \mathcal{H}} |G_{(\cdot,i)}|,~ g_3 = \sum_{i \in \mathcal{K}} |G_{(\cdot,i)}|,~ g_4 = \sum_{i=1}^q |G_{I(\cdot,i)}| \\
		& ~~~~~~~~ \mathcal{H} = \bigg \{ i~ \bigg| ~ \prod_{j=1}^p \big(1-E_{(j,i)} \, \mathrm{mod}~2)\big) = 1 \bigg \}	, ~~ \mathcal{K} = \{1,\dots,h \} \setminus \mathcal{H},
	\end{split}
	\end{equation*}
	where $x \, \mathrm{mod} \,y$, $x,y \in \mathbb{N}_0$ is the modulo operation and $\setminus$ denotes the set difference.
\end{proposition}
\begin{proof}
	We first enclose the polynomial zonotope by a zonotope $\mathcal{Z} \supseteq \mathcal{PZ}$ according to \cite[Prop.~5]{Kochdumper2019}, and then compute an interval enclosure $\mathcal{I} \supseteq \mathcal{Z}$ of this zonotope according to \cite[Prop.~2.2]{Althoff2010a}.
\end{proof}

The core operation for Alg.~\ref{alg:image} is the computation of the image through a quadratic function. While it is possible to obtain the exact image by introducing new dependent factors, we compute a tight enclosure for computational reasons:
\begin{proposition} \label{prop:quadMap}
(Image quadratic function) Given a polynomial zonotope $\mathcal{PZ} = \langle c,G,\linebreak[3] G_I,E\rangle_{PZ} \subset \R$ and a quadratic function $\fitFun(x) = \coeffA \,x^2 + \coeffB \,x + \coeffC$ with $\coeffA,\coeffB,\coeffC,x \in \R$, the image of $\mathcal{PZ}$ through $\fitFun(x)$ can be tightly enclosed by
\begin{equation*}
	\big \{ \fitFun(x)~\big|~ x \in \PZ \big \} \subseteq \langle c_q,G_q,G_{I,q},E_q \rangle_{PZ}
\end{equation*} 
	with 
	\vspace{-0.4cm}
	\begin{equation} \label{eq:quadMap1}
	\begin{split}
		& c_q = \coeffA c^2 + \coeffB c + \coeffC + 0.5 \, \coeffA \sum_{i=1}^q G_{I(\cdot,i)}^2, ~~ G_q = \big[(2 \coeffA c + \coeffB)G ~~ \coeffA \widehat{G} \big], \\
		& E_q =  \big[E ~~ \widehat{E} \big], ~~ G_{I,q} = \big[  (2 \coeffA c + \coeffB)G_I ~~ 2\coeffA \overline{G} ~~\coeffA \widecheck{G} \big],
	\end{split}
	\end{equation}	
	where
	
	\begin{equation} \label{eq:quadMap2}
	\begin{split}
		& \widehat{G} = \big [G^2 ~~ 2\,\widehat{G}_1 ~~ \dots ~~ 2\,\widehat{G}_{h-1}\big ], ~~ \widehat{E} = \big [2 \, E ~~ \widehat{E}_1 ~~ \dots ~~ \widehat{E}_{h-1}\big ], \\
		& \widehat{G}_i = \big [G_{(i)} G_{(i+1)} ~~ \dots ~~ G_{(i)} G_{(h)}\big ], ~~ i = 1,\dots, h-1, \\
		& \widehat{E}_i = \big [E_{(\cdot,i)} + E_{(\cdot, i+1)} ~~ \dots ~~ E_{(\cdot, i)} + E_{(\cdot,h)}\big ], ~~ i = 1,\dots, h-1, \\
		& \overline{G} = \big [G_{(1)} G_I ~~ \dots ~~ G_{(h)} G_I \big], ~~ \widecheck{G} = \big[0.5 \, G_I^2 ~~ 2\,\widecheck{G}_1 ~~ \dots ~~ 2\,\widecheck{G}_{q-1}  \big], \\
		& \widecheck{G}_i = \big[ G_{I(i)} G_{I(i+1)} ~~ \dots ~~ G_{I(i)} G_{I(q)} \big], ~~ i = 1,\dots,q-1,
	\end{split}
	\end{equation}
	and the squares in $G^2$ as well as $G_I^2$ are interpreted elementwise.
\end{proposition}
\begin{proof}
	The proof is provided in Appendix~\ref{sec:appendix}.
\end{proof}

\begin{figure}[!tb]
  \centering
  \psfragfig[width=0.98\columnwidth]{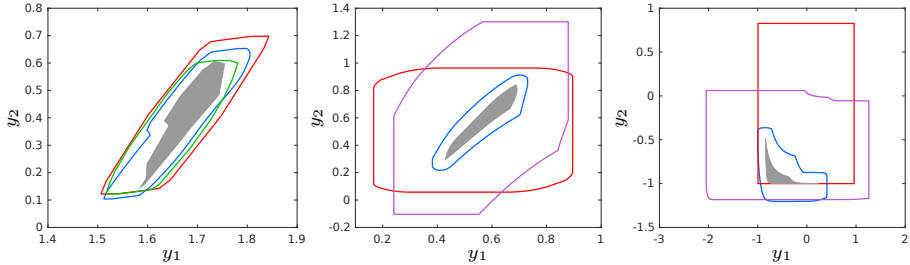}{
  \psfrag{a}[c][c]{\smaller[2]$y_1$}
  \psfrag{b}[c][c]{\rotatebox{180}{\smaller[2]$y_2$}}
  }
  \vspace{-8pt}
  \caption{Image enclosures computed with zonotopes (red), star sets (green), Taylor models (purple), and polynomial zonotopes (blue) for randomly generated neural networks with ReLU activations (left), sigmoid activations (middle), and hyperbolic tangent activations (right). The exact image is shown in gray.}
  \label{fig:image}
\end{figure}


\vspace{\secminus}
\section{Numerical Examples} \label{sec:numEx}
\vspace{-10pt}

We now demonstrate the performance of our approach for image computation, open-loop neural network verification, and reachability analysis of neural network controlled systems. If not stated otherwise, computations are carried out in MATLAB on a 2.9GHz quad-core i7 processor with 32GB memory. We integrated our implementation into CORA \cite{Althoff2015a} and published a repeatability package\footnote{\url{https://codeocean.com/capsule/8237552/tree/v1}}.

\vspace{-1.5pt}
\myparagraph{Image Enclosure}
\vspace{-1.5pt}

\noindent 
First, we demonstrate how our approach captures the non-convexity of the image through a neural network. For visualization purposes we use the deliberately simple example of randomly generated neural networks with two inputs, two outputs, and one hidden layer consisting of 50 neurons. The initial set is $\mathcal{X}_0 = [-1,1] \times [-1,1]$. We compare our polynomial-zonotope-based approach with the zonotope abstraction in \cite{Singh2018}, the star set approach in \cite{Tran2019c} using the triangle relaxation, and the Taylor model abstraction in \cite{Ivanov2021}. While our approach and the zonotope abstraction are applicable to all types of activation functions, the star set approach is restricted to ReLU activations and the Taylor model abstraction is limited to sigmoid and hyperbolic tangent activations. The resulting image enclosures are visualized in Fig.~\ref{fig:image}. While using zonotopes or star sets only yields a convex over-approximation, polynomial zonotopes are able to capture the non-convexity of the image and therefore provide a tighter enclosure. While Taylor models also capture the non-convexity of the image to some extent they are less precise than polynomial zonotopes, which can be explained as follows: 1) The zonotopic remainder of polynomial zonotopes prevents the rapid remainder growth observed for Taylor models, and 2) the quadratic approximation obtained with polynomial regression used for polynomial zonotopes is usually more precise than the Taylor series expansion used for Taylor models.

\addtocounter{footnote}{-1}

\begin{table*}[!tb]
\begin{center}
\caption{Computation times\protect\footnotemark $~$in seconds for different verification tools on a small but representative excerpt of network-specification combinations of the ACAS Xu benchmark. The symbol - indicates that the tool failed to verify the specification.}
\vspace{-18pt}
\label{tab:ASCAS}
\renewcommand{\arraystretch}{1.2}
\begin{tabular}{ l c c c c c c c c c c c c c c c}
 \toprule
 Net.$\,$ & Spec. & \rotatebox{90}{Cgdtest} & \rotatebox{90}{CROWN} &  \rotatebox{90}{Debona} & \rotatebox{90}{ERAN} & \rotatebox{90}{Marabou} & \rotatebox{90}{MN-BaB} & \rotatebox{90}{nnenum} & \rotatebox{90}{nnv} & \rotatebox{90}{NV.jl} & \rotatebox{90}{oval} & \rotatebox{90}{RPM} & \rotatebox{90}{venus2} & \rotatebox{90}{VeriNet} & \rotatebox{90}{Poly. zono.} \\ \midrule
1.9 & 1 & 0.37 & 1.37 & 111 & 3.91 & 0.66 & 48.7 & 0.41 & - & 1.44 & 0.71 & - & 0.53 & 0.55 & \textbf{0.31} \\
2.3 & 4 & - & 0.95 & 1.78 & 1.91 & 0.57 & 12.2 & \textbf{0.06} & - & - & 0.97 & - & 0.46 & 0.17 & 0.16 \\
3.5 & 3 & 0.41 & 0.37 & 1.15 & 1.85 & 0.61 & 6.17 & \textbf{0.05} & - & - & 0.58 & 34.1 & 0.42 & 0.25 & 0.32\\
4.5 & 4 & - & 0.35 & 0.20 & 1.82 & 0.61 & 5.57 & \textbf{0.08} & 0.24 & - & 0.48 & - & 0.42 & 0.21 & 0.16\\
5.6 & 3 & \hspace{1pt}0.38\hspace{1pt} & \hspace{1pt}0.63\hspace{1pt} & \hspace{1pt}2.27\hspace{1pt} & \hspace{1pt}1.82\hspace{1pt} & \hspace{1pt}0.66\hspace{1pt} & \hspace{1pt} 6.51 \hspace{1pt} & \hspace{1pt}\textbf{0.08}\hspace{1pt} & \hspace{7pt}-\hspace{7pt} & \hspace{7pt}-\hspace{7pt} & \hspace{1pt}0.52\hspace{1pt} & \hspace{1pt}40.6\hspace{1pt} & \hspace{1pt}0.48\hspace{1pt} & \hspace{1pt}0.37\hspace{1pt} & \hspace{1pt}0.43\hspace{1pt} \\
 \bottomrule
\end{tabular}
\end{center}
\vspace{-10pt}
\end{table*}

\vspace{-1.5pt}
\myparagraph{Open-Loop Neural Network Verification}
\vspace{-1.5pt}

\noindent For open-loop neural network verification the task is to verify that the image of the neural network satisfies certain specifications that are typically given by linear inequality constraints. We examine the ACAS Xu benchmark from the 2021 and 2022 VNN competition \cite{Bak2021b,Mueller2022} originally proposed in \cite[Sec.~5]{Katz2017}, which features neural networks that provide turn advisories for an aircraft to avoid collisions. All networks consist of 6 hidden layers with 50 ReLU neurons per layer. For a fair comparison we performed the evaluation on the same machine that was used for the VNN competition.
To compute the image through the neural networks with polynomial zonotopes, we apply a quadratic approximation obtained by polynomial regression for the first two layers, and a linear approximation in the remaining layers. Moreover, we recursively split the initial set to obtain a complete verifier. 
The comparison with the other tools that participated in the VNN competition shown in Tab.~\ref{tab:ASCAS} demonstrates that for some verification problems polynomial zonotopes are about as fast as the best tool in the competition. 

\footnotetext{Times taken from \url{https://github.com/stanleybak/vnncomp2021_results} and \url{https://github.com/ChristopherBrix/vnncomp2022_results}.}

\addtocounter{footnote}{-1}

\begin{table*}[!tb]
\begin{center}
\caption{Computation times\protect\footnotemark{} in seconds for reachability analysis of neural network controlled systems considering different tools and benchmarks. The dimension, the number of hidden layers, and the number of neurons in each layer is specified in parenthesis for each benchmark, where $a = 100$, $b = 5$ for ReLU activation functions, and $a = 20$, $b = 3$  otherwise. The symbol - indicates that the tool failed to verify the specification.} \label{tab:reach}
\vspace{-15pt}
\renewcommand{\arraystretch}{1.2}
\begin{tabular}{ l:c c c:c c c c c:c c c c c}
 \toprule
   \multicolumn{1}{c}{} & \multicolumn{3}{c}{\textbf{ReLU}} & \multicolumn{5}{c}{\textbf{sigmoid}} & \multicolumn{5}{c}{\textbf{hyp. tangent}} \\
	 & \rotatebox{90}{Sherlock} & \rotatebox{90}{JuliaReach} & \rotatebox{90}{Poly. zono.} & \rotatebox{90}{Verisig} & \rotatebox{90}{Verisig 2.0} & \rotatebox{90}{ReachNN*} & \rotatebox{90}{POLAR} & \rotatebox{90}{Poly. zono.} & \rotatebox{90}{Verisig} & \rotatebox{90}{Verisig 2.0} & \rotatebox{90}{ReachNN*} & \rotatebox{90}{POLAR} & \rotatebox{90}{Poly. zono.$~$} \\ \midrule 
 B1 (2,\,2,\,20) & $~~~~~~$ & $~~~~~~$ & $~~~~~~$ & - & 49 & 69 & 23 & \textbf{2} & - & 48 & - & 25 & \textbf{8} \\ 
 B2 (2,\,2,\,20) & & & & 12 & 8 & 32 & 10 & \textbf{1} & - & - & - & \textbf{3} & - \\
 B3 (2,\,2,\,20) & & & & 98 & 47 & 130 & 37 & \textbf{3} & 98 & 43 & 128 & 38 & \textbf{3} \\
 B4 (3,\,2,\,20) & & & & 24 & 12 & 20 & 4 & \textbf{1} & 23 & 11 & 20 & 4 & \textbf{1} \\
 B5 (3,\,3,\,100) & & & & \,196\, & 1063 & 31 & 25 & \textbf{2} & - & 168 & - & 31 & \textbf{2} \\
 TORA (4,\,3,\,$a$) & 30 & 2040 & \textbf{13} & 136 & 83 & \hspace{-5pt}13402\hspace{-5pt} & & \textbf{1} & 134 & 70 & 2524 & & \textbf{1} \\
 ACC (6,\,$b$,\,20) & 4 & \textbf{1} & 2 & & & & & & - & 1512 & - & 312 & \textbf{2} \\
 Unicycle (3,\,1,\,500) & 526 & 93 & \textbf{3} & & & & & & & &\\
 Airplane (12,\,3,\,100)~ & - & 29 & \textbf{7} & & & & & & & & \\
 Sin. Pend. (2,\,2,\,25)\, & 1 & 1 & \textbf{1} & $~~~~~~$ & $~~~~~~$ & $~~~~~~$ & $~~~~~$ & $~~~~$ & $~~~~~~$ & $~~~~~~$ & $~~~~~~$ & $~~~~~~$ & $~~~~$\\
 \bottomrule 
\end{tabular}
\end{center}
\end{table*}

\myparagraph{Neural Network Controlled Systems}

\noindent 
The main application of our approach is reachability analysis of neural network controlled systems, for which we now compare the performance to other state-of-the-art tools.
For a fair comparison we focus on published results for which the authors of the tools tuned the algorithm settings by themselves. In particular, we examine the benchmarks from \cite{Schilling2022} featuring ReLU neural network controllers, and the benchmarks from \cite{Ivanov2021} containing sigmoid and hyperbolic tangent neural network controllers. The goal for all benchmarks is to verify that the system reaches a goal set or avoids an unsafe region. 
As the computation times shown in Tab.~\ref{tab:reach} demonstrate, our polynomial-zonotope-based approach is for all but two benchmarks significantly faster than all other state-of-the-art tools, mainly since it avoids all major bottlenecks observed for the other tools: The polynomial approximations of the overall network used by Sherlock and ReachNN* are often imprecise, JuliaReach loses dependencies when enclosing Taylor models by zonotopes, Verisig is quite slow since the nonlinear system used to represent the neural network is high-dimensional, and Verisig 2.0 and POLAR suffer from the rapid remainder growth observed for Taylor models.


\section{Conclusion}
\vspace{-0.25cm}

We introduced a novel approach for computing tight non-convex enclosures of images through neural networks with ReLU, sigmoid, and hyperbolic tangent activation functions. Since we represent sets with polynomial zonotopes, all required calculations can be realized using simple matrix operations only, which makes our algorithm very efficient. While our proposed approach can also be applied to open-loop neural network verification, its main application is reachability analysis of neural network controlled systems. There, polynomial zonotopes enable the preservation of dependencies between the reachable set and the set of control inputs, which results in very tight enclosures of the reachable set. As we demonstrated on various numerical examples, our polynomial-zonotope-based approach consequently outperforms all other state-of-the-art methods for reachability analysis of neural network controlled systems. 

\footnotetext{Computation times taken from \cite[Tab.~1]{Schilling2022} for Sherlock and JuliaReach, from \cite[Tab.~2]{Ivanov2021} for Verisig, Verisig 2.0, and ReachNN*, and from \cite[Tab.~1]{Huang2022} for POLAR.}


\vspace{0.2cm}
{\scriptsize \noindent \textbf{Acknowledgements.} We gratefully acknowledge the financial support from the project justITSELF funded by the European Research Council (ERC) under grant agreement No 817629, from DIREC - Digital Research Centre Denmark, and from the Villum Investigator Grant S4OS. In addition, this material is based upon work supported by the Air Force Office of Scientific Research and the Office of Naval Research under award numbers FA9550-19-1-0288, FA9550-21-1-0121, FA9550-23-1-0066 and N00014-22-1-2156. Any opinions, findings, and conclusions or recommendations expressed in this material are those of the authors and do not necessarily reflect the views of the United States Air Force or the United States Navy.

}


%
\bibliographystyle{splncs04}
\bibliography{kochdumper,cpsGroup}


\begin{subappendices}
\renewcommand{\thesection}{\Alph{section}}%

\section{}
\label{sec:appendix}

We now provide the proof for Prop.~\ref{prop:quadMap}. According to Def.~\ref{def:polyZono}, the one-dimensional polynomial zonotope $\PZ = \langle c,G,G_I,E \rangle_{PZ}$ is defined as
\begin{equation} \label{eq:defPolyZono}
\begin{split}
	\PZ & = \bigg\{ c + \underbrace{\sum _{i=1}^h \bigg( \prod _{k=1}^p \alpha _k ^{E_{(k,i)}} \bigg) G_{(i)}}_{d(\alpha)} + \underbrace{\sum _{j=1}^{q} \beta _j G_{I(j)}}_{z(\beta)} ~ \bigg| ~ \alpha_k, \beta_j \in [-1,1] \bigg\} \\
	& = \big \{ c + d(\alpha) + z(\beta)~\big|~ \alpha,\beta \in [-\mathbf{1},\mathbf{1}] \big \},
\end{split}
\end{equation}
where $\alpha = [\alpha_1~\dots~\alpha_p]^T$ and $\beta = [\beta_1~\dots~\beta_q]^T$. To compute the image through the quadratic function $\fitFun(x)$ we require the expressions $d(\alpha)^2$, $d(\alpha) z(\beta)$, and $z(\beta)^2$, which we derive first. For $d(\alpha)^2$ we obtain
	\begin{equation} \label{eq:expression1}
	\begin{split}
		d(\alpha)^2 & =  \bigg( \sum _{i=1}^h \bigg( \prod _{k=1}^p \alpha _k ^{E_{(k,i)}} \bigg) G_{(i)} \bigg) \bigg( \sum _{j=1}^h \bigg( \prod _{k=1}^p \alpha _k ^{E_{(k,j)}} \bigg) G_{(j)} \bigg) \\
		& = \sum _{i=1}^h \sum_{j=1}^h \bigg( \prod _{k=1}^p \alpha _k ^{E_{(k,i)} + E_{(k,j)}} \bigg) G_{(i)} G_{(j)} \\
		& = \sum_{i=1}^h \bigg( \prod _{k=1}^p \alpha _k ^{2 E_{(k,i)}} \bigg) G_{(i)}^2 + \sum _{i=1}^{h-1} \sum_{j=i+1}^h \bigg( \prod _{k=1}^p \underbrace{\alpha _k ^{E_{(k,i)} + E_{(k,j)}}}_{\alpha_k^{\widehat{E}_{i(k,j)}}} \bigg) 2 \underbrace{G_{(i)} G_{(j)}}_{\widehat{G}_{i(j)}} \\
		& \overset{\substack{\eqref{eq:quadMap2}\\ \vspace{-2pt}}}{=} \sum_{i=1}^{h(h+1)/2} \bigg( \prod _{k=1}^p \alpha _k ^{\widehat{ E}_{(k,i)}} \bigg) \widehat{G}_{(i)},
	\end{split} 
	\end{equation}
	for $d(\alpha)z(\beta)$ we obtain
	\begin{equation} \label{eq:expression2}
	\begin{split}
		d(\alpha)z(\beta) &= \bigg( \sum _{i=1}^h \bigg( \prod _{k=1}^p \alpha _k ^{E_{(k,i)}} \bigg) G_{(i)} \bigg) \bigg( \sum _{j=1}^q \beta_j G_{I(j)} \bigg) \\
		& = \sum_{i=1}^h \sum_{j=1}^q \underbrace{\bigg( \beta_j \prod _{k=1}^p \alpha _k ^{E_{(k,i)}} \bigg)}_{\beta_{q + (i-1)h+j}} G_{(i)} G_{I(j)} \overset{\substack{\eqref{eq:quadMap2} \\ \vspace{-2pt}}}{=} \sum_{i = 1}^{hq} \beta_{q + i} \, \overline{G}_{(i)},
	\end{split}
	\end{equation}
	and for $z(\beta)^2$ we obtain
	\begin{align} 
		z(\beta)^2 & =  \bigg( \sum _{i=1}^q \beta_i G_{I(i)} \bigg) \bigg( \sum _{j=1}^q \beta_j G_{I(j)} \bigg) = \sum _{i=1}^q \sum_{j=1}^q \beta_i \beta_j \, G_{I(i)} G_{I(j)} \nonumber \\
		& = \sum_{i=1}^q \beta_i^2 G_{I(i)}^2 + \sum _{i=1}^{q-1} \sum_{j=i+1}^q \beta_i \beta_j \, 2 \, G_{I(i)} G_{I(j)} \nonumber \\
		& = 0.5 \sum_{i=1}^q G_{I(i)}^2 + \sum_{i=1}^q \underbrace{(2 \beta_i^2 - 1)}_{\beta_{(h+1)q+i}} 0.5 \, G_{I(i)}^2 + \sum _{i=1}^{q-1} \sum_{j=i+1}^q \underbrace{\beta_i \beta_j}_{\beta_{a(i,j)}} 2 \, \underbrace{G_{I(i)} G_{I(j)}}_{\widecheck{G}_{i(j)}} \label{eq:expression3} \\
		& \overset{\substack{\eqref{eq:quadMap2}\\ \vspace{-2pt}}}{=} 0.5 \sum_{i=1}^q G_{I(i)}^2 + \sum_{i=1}^{q(q+1)/2} \beta_{(h+1)q + i} \, \widecheck{G}_{(i)}, \nonumber 
	\end{align}
where the function $a(i,j)$ maps indices $i,j$ to a new index:
\begin{equation*}
	a(i,j) = (h+2)q + j-i + \sum_{k=1}^{i-1} q-k.
\end{equation*}
In \eqref{eq:expression2} and \eqref{eq:expression3}, we substituted the expressions $\beta_j \prod_{k=1}^p \alpha_k^{E_{(k,i)}}$, $2\beta_i^2 -1$, and $\beta_i\beta_j$ containing polynomial terms of the independent factors $\beta$ by new independent factors, which results in an enclosure due to the loss of dependency. The substitution is possible since
\begin{equation*}
\beta_j \prod_{k=1}^p \alpha_k^{E_{(k,i)}} \in [-1,1],~~ 2\beta_i^2 -1 \in [-1,1],~~ \text{and} ~~\beta_i\beta_j \in [-1,1].
\end{equation*}
Finally, we obtain for the image
	\begin{equation*}
	\begin{split}
	& \big \{ \fitFun(x)~\big |~ x \in \mathcal{PZ} \big \} = \big \{\coeffA \,x^2 + \coeffB \,x + \coeffC ~ \big | ~ x \in \mathcal{PZ} \big \} \overset{\substack{\eqref{eq:defPolyZono} \\ \vspace{-2pt}}}{=} \\
	& ~ \\	
	& \big \{ \coeffA ( c + d(\alpha) + z(\beta))^2 + \coeffB (c + d(\alpha) + z(\beta)) + \coeffC   ~\big | ~ \alpha,\beta \in [-\mathbf{1},\mathbf{1}] \big \} = \\
	& ~ \\	
	& \big \{ \coeffA c^2 + \coeffB c + \coeffC  + (2 \coeffA c + \coeffB) d(\alpha) + \coeffA d(\alpha)^2\\
	& ~~ + (2 \coeffA c + \coeffB) z(\beta) + 2 \coeffA d(\alpha)z(\beta) + \coeffA z(\beta)^2  ~ \big | ~ \alpha,\beta \in [-\mathbf{1},\mathbf{1}] \big\} \overset{\substack{\eqref{eq:expression1},\eqref{eq:expression2},\eqref{eq:expression3}\\ \vspace{-2pt}}}{\subseteq} \\
	& ~ \\
	& \bigg \langle \coeffA c^2 + \coeffB c + \coeffC + 0.5 \, \coeffA \sum_{i=1}^q G_I^2, \big[(2 \coeffA c + \coeffB)G ~~ \coeffA \widehat{G} \big], \\
	& ~~~~~~~~~~~~~~~~~~~~~ \big[  (2 \coeffA c + \coeffB)G_I ~~ 2\coeffA \overline{G} ~~\coeffA \widecheck{G} \big], \big[E ~~ \widehat{E} \big] \bigg \rangle_{PZ} \hspace{-8pt} \overset{\substack{ \eqref{eq:quadMap1} \\ \vspace{-2pt}}}{=} \langle c_q,G_q,G_{I,q},E_q \rangle_{PZ},
	\end{split}
	\end{equation*}
which concludes the proof.
\end{subappendices}

\end{document}